\documentclass[journal]{IEEEtran}

\usepackage{cite}
\usepackage{color}
\usepackage{amsmath}
\usepackage{bm}
\usepackage{amsfonts}
\usepackage{epsfig}
\usepackage{graphicx}
\usepackage{amssymb}
\usepackage{multirow}
\usepackage{algpseudocode}
\usepackage{booktabs}
\usepackage{epstopdf}
\usepackage{amsthm}
\newtheorem{theorem}{Theorem}
\newtheorem{lemma}{Lemma}
\usepackage{array}
\usepackage{url}
\usepackage[ruled]{algorithm2e}
\usepackage{enumitem}

\ifCLASSINFOpdf
\else
\fi
\begin{document}
\title{Robust and Efficient Fuzzy C-Means Clustering Constrained on Flexible Sparsity}
\author{Jinglin Xu, Junwei Han$^\star$~\IEEEmembership{Senior Member,~IEEE,} Mingliang Xu,
        Feiping Nie and Xuelong Li~\IEEEmembership{Fellow,~IEEE,}% <-this % stops a space
\IEEEcompsocitemizethanks{\IEEEcompsocthanksitem Jinglin Xu and Junwei Han are with the School of Automation, Northwestern Polytechnical University, Xi'an 710072, Shaanxi, China. E-mail: \{xujinglinlove, junweihan2010\}@gmail.com
\IEEEcompsocthanksitem Mingliang Xu is with the School of Information Engineering of Zhengzhou University, China. E-mail: iexumingliang@zzu.edu.cn
\IEEEcompsocthanksitem Feiping Nie and Xuelong Li are with School of Computer Science and Center for OPTical IMagery Analysis and Learning (OPTIMAL), Northwestern Polytechnical University, Xi'an 710072, Shaanxi, China. E-mail: feipingnie@gmail.com, xuelong$\_$li@nwpu.edu.cn
\IEEEcompsocthanksitem$^\star$Corresponding author}}

\maketitle

\begin{abstract}
Clustering is an effective technique in data mining to group a set of objects in terms of some attributes. Among various clustering approaches, the family of K-Means algorithms gains popularity due to simplicity and efficiency. However, most of existing K-Means based clustering algorithms cannot deal with outliers well and are difficult to efficiently solve the problem embedded the $L_0$-norm constraint. To address the above issues and improve the performance of clustering significantly, we propose a novel clustering algorithm, named REFCMFS, which develops a $L_{2,1}$-norm robust loss as the data-driven item and imposes a $L_0$-norm constraint on the membership matrix to make the model more robust and sparse flexibly. In particular, REFCMFS designs a new way to simplify and solve the $L_0$-norm constraint without any approximate transformation by absorbing $\|\cdot\|_0$ into the objective function through a ranking function. These improvements not only make REFCMFS efficiently obtain more promising performance but also provide a new tractable and skillful optimization method to solve the problem embedded the $L_0$-norm constraint. Theoretical analyses and extensive experiments on several public datasets demonstrate the effectiveness and rationality of our proposed REFCMFS method.
\end{abstract}


\IEEEpeerreviewmaketitle

\section{Introduction}

As a fundamental problem in machine learning, clustering is widely used for many fields, such as the network data (including Protein-Protein Interaction Networks \cite{bhowmick2016clustering}, Road Networks \cite{han2017systematic}, Geo-Social Network \cite{wu2018density}), medical diagnosis \cite{nithya2013survey}, biological data analysis \cite{wiwie2015comparing}, environmental chemistry \cite{d2015time} and so on. K-Means clustering is one of the most popular techniques because of its simplicity and effectiveness, which randomly initializes the cluster centroids, assigns each sample to its nearest cluster and then updates cluster centroid iteratively to cluster a dataset into some subsets.

Over the past years, many modified versions of K-Means algorithms have been proposed, such as Global K-Means \cite{LIKAS2003451} and its variants \cite{BAGIROV20083192,LAI20101954,BAGIROV2011866}, MinMax K-Means clustering \cite{Tzortzis2014The}, K-Means based Consensus clustering \cite{wu2015k}, Optimized Cartesian K-Means \cite{wang2015optimized}, Group K-Means \cite{wang2015group}, Robust K-Means \cite{NIPS2016_6126}, I-K-Means-+ \cite{Ismkhan2018I} and so on. Most importantly, researchers have pointed out that the objective function of K-Means clustering can be expressed as the Frobenius norm of the difference between the data matrix and the low rank approximation of that data matrix \cite{arora2013similarity,bauckhage2015k}.
Specifically, the problem of hard K-Means clustering is as follows:
%----------------------------
\begin{equation}\label{kmeans}
\underset{\bm{B},\bm{\alpha}}{\min}\sum_{i=1}^n\sum_{k=1}^c\alpha_{ik}\|\bm{x}_i-\bm{b}_k\|^2
=\underset{\bm{B},\bm{\alpha}}{\min}\|\bm{X}-\bm{B}\bm{\alpha}^\top\|^2
\end{equation}
%-----------------------------
where $\bm{X}\in\mathbb{R}^{d\times n}$ is a matrix of data vectors $\bm{x}_i\in\mathbb{R}^d$; $\bm{B}\in\mathbb{R}^{d\times c}$ is a matrix of cluster centroids $\bm{b}_k\in\mathbb{R}^d$; $\bm{\alpha}\in\mathbb{R}^{n\times c}$ is a cluster indicator matrix of binary variables such that $\alpha_{ik}=1$ if $\bm{x}_i\in\!C_k$ where $C_k$ denotes the $k$-th cluster and otherwise $\alpha_{ik}=0$.

Although the K-Means clustering algorithm has been used widely, it is sensitive to the outliers, which easily deteriorate the clustering performance. Therefore, two main approaches are proposed to deal with the outliers in K-Means clustering: one based on outlier analysis (outlier detection or removal), and the other one based on outlier suppression (robust model).

For the first one, much work has been done on outlier analysis. Several algorithms\cite{jiang2001two,he2003discovering,hautamaki2005improving,jiang2008clustering,zhou2009novel,pamula2011outlier,jiang2016initialization} perform clustering and detect outliers separately in different stages, where the dataset is divided into different clusters that can be used to identify outliers by measuring the data points and clusters. Besides, \cite{rehm2007novel} provides the definition of outliers according to the noise distance or the remote distances between the data points and all other clustering centers. \cite{zhang2009new} introduces a local distance-based outlier factor to measure the outlierness of objects in scattered datasets. \cite{ott2014integrated} proposes a sub-gradient-based algorithm to jointly solve the problems of clustering and outliers detection. \cite{whang2015non} proposes a non-exhaustive overlapping K-Means algorithm to identify outliers during the clustering process. \cite{gan2017k} performs the clustering and outlier detection simultaneously by introducing an additional \lq\lq cluster\rq\rq\ into the K-Means algorithm to hold all outliers. For the second one, the main strategy of outlier suppression is to modify the objective functions as a robust model, such as \cite{Lam2004Robust}.

Fuzzy C-Means (FCM) introduces the concept of fuzzy sets that has been successfully used in many areas, which makes the clustering method more powerful. However, FCM still exists several drawbacks, including the sensitivity to initialization and outliers, and the limitation to convex clusters. Therefore, many extensions and variants of FCM clustering are advanced in recent years. Augmented FCM \cite{izakian2013clustering} revisits and augments the algorithm to make it applicable to spatiotemporal data. Suppressed FCM \cite{szilagyi2014generalization} increases the difference between high and low membership grades and gives more accurate partitions of the data with less iterations compared to the FCM. Sparse FCM \cite{qiu2015sparse} reforms traditional FCM to deal with high dimensional data clustering, based on Witten's sparse clustering framework. Kernel based FCM \cite{ding2016kernel} optimizes FCM, based on the genetic algorithm optimization which is combined of the improved genetic algorithm and the kernel technique to optimize the initial clustering center firstly and to guide the categorization. Multivariate FCM \cite{pimentel2016multivariate} proposes two multivariate FCM algorithms with different weights aiming to represent how important each different variable is for each cluster and to improve the clustering quality. Robust-Learning FCM \cite{Yang2017Robust} is free of the fuzziness index and initializations without parameter selection, and can also automatically find the best number of clusters.

Since the above extensions still has weak performance when dealing with outliers, several robust FCM algorithms come out. Specifically, conditional spatial FCM \cite{adhikari2015conditional} improves the robustness of FCM through the incorporation of conditioning effects imposed by an auxiliary variable corresponding to each pixel.
Modified possibilistic FCM \cite{aparajeeta2016modified} jointly considers the typicality as well as the fuzzy membership measures to model the bias field and noise. Generalized entropy-based possibilistic FCM \cite{askari2017generalized} utilizes the functions of distance in the fuzzy, possibilistic, and entropy terms of the clustering objective function to decrease the noise effects on the cluster centers.
Fast and robust FCM \cite{lei2018significantly} proposes a significantly faster and more robust based on the morphological reconstruction and membership filtering.

Inspired by above analysis, we develop a novel clustering method in this work, named as Robust and Efficient Fuzzy C-Means Clustering constrained on Flexible Sparsity (REFCMFS), by introducing a flexible sparse constraint imposed on the membership matrix to improve the robustness of the proposed method and to provide a new idea to simplify solving the problem of sparse constraint on $L_0$-norm. The proposed method REFCMFS not only improves the robustness from the $L_{2,1}$-norm data-driven term but also obtains the solution with proper sparsity and greatly reduces the computational complexity.

Note that we have proposed a Robust and Sparse Fuzzy K-Means (RSFKM) clustering method \cite{xu2016robust} recently. However, our proposed REFCMFS method in this paper is quite different from RSFCM. Concretely, RSFCM takes into account the robustness of the data-driven term by utilizing $L_{2,1}$-norm and capped $L_1$-norm, and utilizes the Lagrangian Multiplier method and Newton method to solve the membership matrix whose sparseness is adjusted by the regularized parameter.
In contrast, our proposed REFCMFS method maintains the robustness of the clustering model by using the $L_{2,1}$-norm loss and introduces the sparse constraint $L_0$-norm imposed on the membership matrix with a flexible sparsity, i.e., $\|\bm{\alpha}\|_0=K$ where $K\in\mathbb{R}$ denotes the number of nonzero elements in $\bm{\alpha}$.
It is well known that solving the problem with $L_0$-norm constraint is difficult. The proposed REFCMFS method absorbs this constraint into the objective function by designing a novel ranking function $\psi$, which is an efficient way to calculate the optimal membership matrix and greatly reduces the computational complexity, especially for a large dataset. The related theoretical analyses and comparison experiments can be demonstrated in Sections \ref{theoretical_analyses} and \ref{experiments}.

The contributions of our proposed REFCMFS method in this paper can be summarized as follows:
\begin{enumerate}
\item REFCMFS develops the $L_{2,1}$-norm loss for the data-driven item and introduces the $L_0$-norm constraint on the membership matrix, which makes the model have the abilities of robustness, proper sparseness, and better interpretability. This not only avoids the incorrect or invalid clustering partitions from outliers but also greatly reduces the computational complexity.
\item REFCMFS designs a new way to simplify and solve the $L_0$-norm constraint directly without any approximation. For each instance, we absorb $\|\cdot\|_0$ into the objective function through a ranking function which sorts $c$ elements in ascending order and selects first $\tilde{\kappa}$ smallest elements as well as corresponding membership values and sets the rest of membership values as zeros. This makes REFCMFS can be solved by a tractable and skillful optimization method and guarantees the optimality and convergence.
\item Theoretical analyses, including the complexity analysis and convergence analysis, are presented briefly, and extensive experiments on several public datasets demonstrate the effectiveness and rationality of the proposed REFCMFS method.
\end{enumerate}
The rest of this paper is organized as follows. The related works are introduced in Section 2. In Section 3, we develop a novel REFCMFS method and provide a new idea to solve it masterly. Some theoretical analyses of REFCMFS, i.e., complexity analysis and convergence analysis, are shown in Section 4. Section 5 provides the experimental results on several public datasets, followed by convergence curves and parameter sensitivity analyses of experiments. The conclusion is shown in Section 6.

\section{Preliminary Knowledge}

In this section, we briefly review some typical literature on K-Means Clustering, Gaussian Mixed Model, Spectral Clustering, Fuzzy C-Means Clustering, and Robust and Sparse Fuzzy K-Means Clustering related to the proposed methods.

\subsection{K-Means Clustering}
The K-Means clustering has been shown in problem (\ref{kmeans}), where $\bm{\alpha}\in\mathbb{R}^{n\times c}$ is the membership matrix and each row of $\bm{\alpha}$ satisfies the $1$-of-$c$ coding scheme (if a data point $\bm{x}_i$ is assigned to the $k$-th cluster then $\alpha_{ik}=1$ and $\alpha_{ik}=0$ otherwise). Although K-Means Clustering is simple and can be solved efficiently, it is very sensitive to outliers.

\subsection{Fuzzy C-Means Clustering}
As one of the most popular fuzzy clustering techniques, Fuzzy C-Means Clustering \cite{bezdek1980convergence} is to minimize the following objective function:
%----------------------------------------
\begin{equation}
\begin{split}
&\underset{\bm{\alpha},\bm{B}}{\min}\sum_{i=1}^n\sum_{k=1}^c\|\bm{x}_i-\bm{b}_k\|_2^2\alpha_{ik}^r\\
&s.t.\sum_{k=1}^c\alpha_{ik}=1,0<\sum_{i=1}^n\alpha_{ik}<n,\alpha_{ik}\geq0
\end{split}
\label{fcm}
\end{equation}
%----------------------------------------
where $\bm{\alpha}\in\mathbb{R}^{n\times c}$ is the membership matrix and whose elements are nonnegative and their sum equals to one on each row. The parameter $r>1$ is a weighting exponent on each fuzzy membership and determines the amount of fuzziness of the resulting clustering.

The objective functions of K-Means and FCM are virtually identical, and the only difference is to introduce a vector (i.e., each row of $\bm{\alpha}$) which expresses the percentage of belonging of a given point to each of the clusters. This vector is submitted to a 'stiffness' exponent (i.e., $r$) aimed at providing more importance to the stronger connections (and conversely at minimizing the weight of weaker ones). When $r$ tends towards infinity, the resulting vector becomes a binary vector, hence making the objective function of FCM identical to that of K-Means. Besides, FCM tends to run slower than K-Means since each point is evaluated with each cluster, and more operations are involved in each evaluation. K-Means only needs to do a distance calculation, whereas FCM needs to do a full inverse-distance weighting.

\subsection{Gaussian Mixed Model}
Unlike K-Means clustering which generates hard partitions of data, Gaussian Mixed Model (GMM) \cite{bishop2006pattern} as one of the most widely used mixture models for clustering can generate soft partition and is more flexible. Considering that each cluster can be mathematically represented by a parametric distribution, the entire dataset is modeled by a mixture of these distributions. In GMM, the mixture of Gaussians has $c$ component densities $p_k|_{k=1}^c$ mixed together with $c$ mixing coefficients $\pi_k|_{k=1}^c$:
%--------------------------------------
\begin{equation}
P(\bm{x}_i|\Theta)=\sum_{k=1}^c\pi_kp_k(\bm{x}_i|\theta_k)
\label{gmm}
\end{equation}
%--------------------------------------
where $\Theta\!=\!(\pi_1,\!\cdots\!,\pi_c,\theta_1,\!\cdots\!,\theta_c)$ are parameters such that $\sum_{k=1}^c\pi_k\!=\!1$ and each $p_k$ is a Gaussian density function parameterized by $\theta_k$. GMM use mixture distributions to fit the data and the conditional probabilities of data points, $P(\bm{x}_i|\Theta)|_{i=1}^n$, are used to assign probabilistic labels.
Although the Expectation-Maximization (EM) algorithm for GMM can achieve the promising results, it has a high computational complexity.

\subsection{Spectral Clustering}
The general Spectral Clustering (SC) method \cite{chung1997spectral} needs to construct an adjacency matrix and calculate the eigen-decomposition of the corresponding Laplacian matrix. However, both of these two steps are computational expensive. Given a data matrix $\bm{X}\in\mathbb{R}^{d\times n}$, spectral clustering first constructs an undirected graph by its adjacency matrix $\bm{W}\!=\!\{w_{ij}\}|_{i,j=1}^n$, each element of which denotes the similarity between $\bm{x}_i$ and $\bm{x}_j$. The graph Laplacian $\bm{L}\!=\!\bm{D}\!-\!\bm{W}$ is calculated where $\bm{D}$ denotes the degree matrix which is a diagonal matrix whose entries are row sums of $\bm{W}$, i.e., $d_{ii}\!=\!\sum_{j=1}^nw_{ij}$. Then spectral clustering use the top $c$ eigenvectors of $\bm{L}$ corresponding to the $c$ smallest eigenvalues as the low dimensional representations of the original data. Finally, the traditional K-Means clustering is applied to obtained the clusters. Due to the high complexity of the graph construction and the eigen-decomposition, spectral clustering is not suitable to deal with the large-scale applications.

\subsection{Robust and Sparse Fuzzy K-Means Clustering}
Considering that the $L_2$-norm loss imposed in problems (\ref{kmeans}) and (\ref{fcm}) lacks of robustness, with the development of $L_{2,1}$-norm \cite{nie2010efficient,jiang2015robust} technologies, amount of robust loss functions are designed and shown their empirical successes in various applications. For example, the recent work \cite{xu2016robust} provided a robust and sparse Fuzzy K-Means clustering by introducing two robust loss functions (i.e., $L_{2,1}$-norm and capped $L_1$-norm) and a penalized regularization on membership matrix. Its objective functions can be written as:
%---------------------------------------
\begin{equation}
\begin{split}
&\underset{\bm{\alpha},\bm{B}}{\min}\sum_{i=1}^n\sum_{k=1}^c\tilde{d}_{ik}\alpha_{ik}+\gamma\|\bm{\alpha}\|_F^2\\
&s.t. \bm{\alpha}\bm{1}_c=\bm{1}_n,\bm{\alpha}\geq\bm{0}
\end{split}
\label{rsfkm}
\end{equation}
%---------------------------------------
where
%---------------------------------------
\begin{equation}
\tilde{d}_{ik}=\|\bm{x}_i-\bm{b}_k\|_2,\ \text{or}\ \tilde{d}_{ik}=\min(\|\bm{x}_i-\bm{b}_k\|_2,\varepsilon)
\label{rsfkm-d_ik}
\end{equation}
%----------------------------------------
where $\bm{\alpha}\in\mathbb{R}^{n\times c}$ is the membership matrix and $\gamma$ is the regularization parameter, and $\varepsilon$ is a threshold. When $\gamma$ is zero, the membership vector of each sample becomes extremely sparse (only one element is nonzero and others are zero). The membership matrix equals to the binary clustering indicator matrix, which is hard K-Means clustering. With the gradual increase of $\gamma$, membership vector contains a growing number of nonzero elements. When $\gamma$ becomes a large value, all elements in membership vectors are all nonzero, which is equivalent to FCM clustering.

\section{Robust and Efficient FCM Clustering Constrained on Flexible Sparsity}

In this section, we introduce our proposed REFCMFS method, which develops the $L_{2,1}$-norm loss for the data-driven item and imposes the $L_0$-norm constraint on the membership matrix to make the model more robust and sparse flexibly. We also design a new way to simplify and solve the $L_{2,1}$-norm loss with the $L_0$-norm constraint efficiently without any approximation.

\subsection{Formulation}

Based on the Fuzzy C-Means Clustering algorithm, in order to make the model more robust, proper sparse, and efficient during the clustering, we propose the following objective function:
%----------------------------------------
\begin{equation}
\begin{split}
&\underset{\bm{\alpha},\bm{B}}{\min}\sum_{i=1}^n\sum_{k=1}^c\|\bm{x}_i-\bm{b}_k\|_2\alpha_{ik}^r \\
&s.t.\bm{\alpha}\geq\bm{0},\bm{\alpha}\bm{1}_c=\bm{1}_n,\|\bm{\alpha}\|_0=K
\end{split}
\label{our_objective_function}
\end{equation}
%-----------------------------------------
where $\bm{\alpha}\in\mathbb{R}^{n\times c}$ is the membership matrix constrained by the $L_0$-norm and $r>1$ is the hyper-parameter that controls how fuzzy the cluster will be (the higher, the fuzzier) and $K\in\mathbb{N}_+$ denotes the number of nonzero elements in $\bm{\alpha}$, which constrains the sparseness of membership matrix to be $K$.

We find that $\|\bm{\alpha}\|_0\!=\!K$ constrains the number of all the elements of $\bm{\alpha}$ not the number of elements of each $\bm{\alpha}^i$, where $\bm{\alpha}^i$ is the $i$-th row of membership matrix $\bm{\alpha}$ and responds to the membership vector of the $i$-th sample. This easily leads to two extreme cases for $\bm{\alpha}^i$, i.e., $\bm{\alpha}^{em1}\!=\![0,\cdots,0,1,0,\cdots,0]$ and $\bm{\alpha}^{em2}\!=\![\frac{1}{c},\cdots,\frac{1}{c}]$, where $\bm{\alpha}^{em1}$ makes the soft partition degrade into the hard partition and $\bm{\alpha}^{em2}$ results in an invalid partition for the $j$-th sample because all the membership values are equal. Therefore, we further divide the problem (\ref{our_objective_function}) into $n$ subproblems and impose the $L_0$-norm constraint on the membership vector for each sample. Therefore, REFCMFS can be presented as follows:
%------------------------------------------------------------
\begin{equation}
\begin{split}
&\underset{\bm{\alpha}^i,\bm{B}}{\min}\sum_{k=1}^c\|\bm{x}_i-\bm{b}_k\|_2\alpha_{ik}^r\\
&s.t.\bm{\alpha}^i\geq\bm{0},\bm{\alpha}^i\bm{1}_c=1,\|\bm{\alpha}^i\|_0=\tilde{K}
\end{split}
\label{subproblem_each_sample}
\end{equation}
%---------------------------------------------------------------
where $\tilde{K}\!=\!\frac{K}{n}$ denotes the number of nonzero elements in $\bm{\alpha}^i$, $\tilde{K}\in\mathbb{N}_+$, and $1<\tilde{K}<c$.

It is obvious that $\|\bm{x}_i\!-\!\bm{b}_k\|_2$ achieves the robustness by using the $L_{2,1}$-norm on the similarity between $\bm{x}_i$ and $\bm{b}_k$, and $\|\bm{\alpha}^i\|_0\!=\!\tilde{K}$ makes the membership vector with the sparsity $\tilde{K}$, which not only avoids the incorrect or invalid clustering partitions from outliers but also greatly reduces the computational complexity.

\subsection{Optimization}
In this subsection, we provide an efficient iterative method to solve problem (\ref{subproblem_each_sample}). More specifically, we alternatively update one optimization variable while keeping other optimization variables fixed. It is represented as follows.

\paragraph{Step 1: Solving $\bm{\alpha}$ while fixing $\bm{B}$}
With the centroid matrix $\bm{B}$ fixed, the problem (\ref{subproblem_each_sample}) becomes:
%-----------------------------------------
\begin{equation}
\begin{split}
&\underset{\bm{\alpha}^i}{\min}\sum_{k=1}^c\|\bm{x}_i-\bm{b}_k\|_2\alpha_{ik}^r\\
&s.t.\bm{\alpha}^i\geq\bm{0},\bm{\alpha}^i\bm{1}_c=1,\|\bm{\alpha}^i\|_0=\tilde{K}
\end{split}
\label{alpha_i}
\end{equation}
%-----------------------------------------

Due to directly solve the problem (\ref{alpha_i}) difficultly, we need to do some transformations as follows:
%---------------------------------------
\begin{equation}
\begin{split}
&\underset{\bm{\alpha}^i}{\min}\sum_{k=1}^ch_{ik}\alpha_{ik}^r\\
&s.t.\bm{\alpha}^i\geq\bm{0},\bm{\alpha}^i\bm{1}_c=1, \|\bm{\alpha}^i\|_0=\tilde{K}
\end{split}
\label{alpha_i_h}
\end{equation}
%----------------------------------------
where $h_{ik}\!=\!\|\bm{x}_i\!-\!\bm{b}_k\|_2\in\mathbb{R}$ and $\bm{h}^i=[h_{i,1},\cdots,h_{i,c}]\in\mathbb{R}^c$ is a row-vector contained different $h_{ik}|_{k=1}^c$.
To efficiently minimize the problem (\ref{alpha_i_h}), we define a ranking function $\psi$ and perform it on $\bm{h}^i$, and then obtain:
%-----------------------------------
\begin{equation}
\psi(\bm{h}^i)=\bm{h}^i\bm{P}=[h_{i,\psi(1)},\!\cdots\!,h_{i,\psi(\tilde{K})},\!\cdots\!,h_{i,\psi(c)}]
\label{permutation}\end{equation}
%------------------------------------
where $\psi$ sorts $c$ elements of $\bm{h}^i$ in ascending order and $\bm{P}$ is the corresponding permutation matrix which results in permuting columns of $\bm{h}^i$ along the order $\{\psi(1),\psi(2),\cdots,\psi(c)\}$. Based on equation (\ref{permutation}), we select first $\tilde{K}$ smallest elements as well as their corresponding membership values in $\bm{\alpha}^i$, meanwhile, setting the membership values of the rest $c\!-\!\tilde{K}$ elements as zeros, i.e., $\alpha_{i,\psi(k)}|_{k=\tilde{K}+1}^c\!=\!0$. Intuitively, we present the above operations in Figure \ref{explanation}.
%---------------------------------------------
\begin{figure}[t]
\begin{center}
\includegraphics[width=\linewidth]{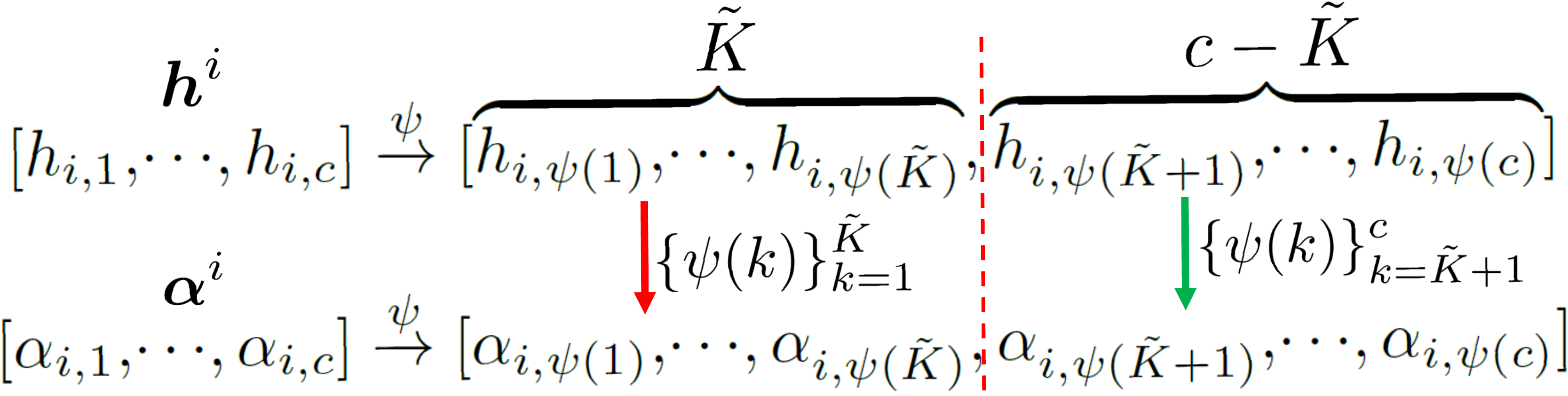}
\end{center}
\caption{Explaination of performing $\psi$ on $\bm{h}^i$ and $\bm{\alpha}^i$. For instance, supposed that $c\!=\!5$, $\tilde{K}\!=\!3$, and $\bm{h}^i\!=\![h_{i,1},h_{i,2},h_{i,3},h_{i,4},h_{i,5}]\!=\![2.4, 3.5, 0.6, 7.8, 1.9]$, then $\psi(\bm{h}^i)\!=\![h_{i,\psi(1)},h_{i,\psi(2)},h_{i,\psi(3)},h_{i,\psi(4)},h_{i,\psi(5)}]\!=\![0.6, 1.9, 2.4, 3.5, 7.8]$. According to the first $\tilde{K}$ elements of $\psi(\bm{h}^i)$, i.e., [0.6, 1.9, 2.4], optimizing their corresponding membership values and setting the rest membership values as zeros, there is $\bm{\alpha}^i=[\alpha_{i,\psi(1)},\alpha_{i,\psi(2)},\alpha_{i,\psi(3)},\alpha_{i,\psi(4)},\alpha_{i,\psi(5)}]=[0.39, 0.33, 0.18, 0, 0]$.}
\label{explanation}
\end{figure}
%---------------------------------------------

Therefore, the problem (\ref{alpha_i_h}) is equivalent to the following problem by absorbing the $L_0$-norm constraint $\|\bm{\alpha}^i\|_0=\tilde{K}$ into the objective function:
%------------------------------------------
\begin{equation}
\begin{split}
&\underset{\alpha_{i,\psi(k)}|_{k=1}^{\tilde{K}}}{\min}\sum_{k=1}^{\tilde{K}}h_{i,\psi(k)}\alpha_{i,\psi(k)}^r\\
&s.t.\alpha_{i,\psi(k)}\geq0,\sum_{k=1}^{\tilde{K}}\alpha_{i,\psi(k)}=1
\end{split}
\label{alpha_i_h_new}
\end{equation}
%------------------------------------------

By using the Lagrangian Multiplier method, the Lagrangian function of problem (\ref{alpha_i_h_new}) is:
%-----------------------------------------
\begin{equation}
\mathcal{L}(\alpha_{i,\psi(k)},\lambda)=\sum_{k=1}^{\tilde{K}}h_{i,\psi(k)}\alpha_{i,\psi(k)}^r+\lambda(\sum_{k=1}^{\tilde{K}}\alpha_{i,\psi(k)}-1)
\label{lagrangian_alpha}
\end{equation}
%-----------------------------------------
where $\lambda$ is the  Lagrangian multiplier.
To solve the minimum of problem (\ref{lagrangian_alpha}), we take the derivatives of $\mathcal{L}$ with respect to $\alpha_{i,\psi(k)}$ and $\lambda$, respectively, and set them to zeros. We obtain the optimal solution of problem (\ref{alpha_i_h_new}):
%-----------------------------------------
\begin{equation}
\alpha_{i,\psi(k)}=\frac{h_{i,\psi(k)}^{\frac{1}{1-r}}}{\sum_{s=1}^{\tilde{K}}h_{i,\psi(s)}^{\frac{1}{1-r}}}
\label{solution_alpha}
\end{equation}
%-----------------------------------------
where $k\!=\!1,\cdots,\tilde{K}$.

Substituting the equation (\ref{solution_alpha}) into problem (\ref{alpha_i_h_new}), its optimal value arrives at:
%------------------------------------------
\begin{equation}
\begin{split}
\sum_{k=1}^{\tilde{K}}h_{i,\psi(k)}\alpha_{i,\psi(k)}^r
&=(\sum_{k=1}^{\tilde{K}}h_{i,\psi(k)}^{\frac{1}{1-r}})^{1-r}   \\
&=\frac{1}{(\sum_{k=1}^{\tilde{K}}(\frac{1}{h_{i,\psi(k)}})^{\frac{1}{r-1}})^{r-1}}
\end{split}
\label{4-8}
\end{equation}
%-------------------------------------------
It is obvious that the minimum depends on $h_{i,\psi(k)}$, the smaller the better.

Therefore, the optimal solution of problem (\ref{alpha_i_h}) is:
%--------------------------------------
\begin{equation}
\begin{split}
\alpha_{i,\psi(k)}=\left\{
             \begin{array}{lr}
             \frac{h_{i,\psi(k)}^{\frac{1}{1-r}}}{\sum_{s=1}^{\tilde{K}}h_{i,\psi(s)}^{\frac{1}{1-r}}} & k=1,\!\cdots\!,\tilde{K}\\
             0 & k=\tilde{K}+1,\!\cdots\!,c
             \end{array}
\right.
\label{final_solution_alpha}
\end{split}
\end{equation}
%-----------------------------------------------

\paragraph{Step 2: Solving $\bm{B}$ while fixing $\bm{\alpha}$}
With the membership matrix $\bm{\alpha}$ fixed, the problem (\ref{subproblem_each_sample}) becomes:
%----------------------------------------------
\begin{equation}
\underset{\bm{B}}{\min}\sum_{i=1}^n\sum_{k=1}^c\|\bm{x}_i-\bm{b}_k\|_2\alpha_{ik}^r
\label{B}
\end{equation}
%------------------------------------------------
which can be solved by introducing a nonnegative auxiliary variable $s_{ik}$ and using the iterative re-weighted method. Thus, we rewrite the problem (\ref{B}) as:
%--------------------------------------
\begin{equation}
\underset{\bm{B}}{\min}\sum_{i=1}^n\sum_{k=1}^cs_{ik}\|\bm{x}_i-\bm{b}_k\|_2^2\alpha_{ik}^r
\label{rewrite_B}\end{equation}
%----------------------------------------
where
%-----------------------
\begin{equation}
s_{ik}=\frac{1}{2\|\bm{x}_i-\bm{b}_k\|_2}
\label{s_ik}\end{equation}
%-------------------------
The optimal solution of problem (\ref{rewrite_B}) can be reached by taking derivative and setting it to zero. That is:
%---------------------------------------------------------
\begin{align}
\bm{b}_k=\frac{\sum_{i=1}^n\bm{x}_is_{ik}\alpha_{ik}}{\sum_{i=1}^ns_{ik}\alpha_{ik}}
\label{solution_bk}
\end{align}
%--------------------------------------------------------
Assuming that $\bm{\alpha}^{(t)}$ and $\bm{B}^{(t)}$ are computed at the $t$-th iteration, we can update the nonnegative auxiliary variable $s_{ik}$ according to equation (\ref{s_ik}) by current $\bm{B}^{(t)}$.
Intuitively, the above optimization is summarized in Algorithm \ref{alg2}.
\begin{algorithm}[t]
\SetAlgoLined
        \caption{Solving the problem (\ref{subproblem_each_sample})}
        \KwIn{Data matrix $\bm{X}\in\mathbb{R}^{d\times\!n}$, the number of clusters $c$, parameters $r$ and $\tilde{K}$}
        \KwOut{Membership matrix $\bm{\alpha}$, centroid matrix $\bm{B}$}

        Initialize centroid matrix $\bm{B}$ \;
        \While{$\mathit{obj}(t\!-\!1)\!-\!\mathit{obj}(t)=\mathit{thresh}\leq10^{-7}$}
        {
            \For{each sample $i$, $1\leq i\leq n$}
            {Obtaining the membership values $\alpha_{ik}$ in the problem (\ref{alpha_i}) by using the equation (\ref{final_solution_alpha})}
            \For{each cluster $k$, $1\leq k\leq c$}
            {Calculating the centroid vector $\bm{b}_k$ and updating the auxiliary variable $s_{ik}$ via (\ref{solution_bk}) and (\ref{s_ik}).}
    }
\label{alg2}
\end{algorithm}

\section{Theoretical Analysis}\label{theoretical_analyses}

 In this section, we provide computational complexity analysis and convergence analysis of our proposed REFCMFS method.

\subsection{Computational Analysis}\label{computational_analysis}

Suppose we have $n$ samples in $c$ clusters and each sample has $d$ dimensions. For each iteration, the computational complexity of REFCMFS involves two steps. The first step is to compute the membership matrix $\bm{\alpha}$, which has computational complexity $O(ncd+nc^2+nc\tilde{K})$. The second step is to calculate the centroid matrix $B$, which needs to $O(dnc)$ operations. For several public datasets $\tilde{K}<c$, the computational complexity of REFCMFS for each iteration is $O(nc\cdot\max(d,c))$. In addition, the computational complexities of other typical methods are listed in Table \ref{computational_complexity}, where $O(M(c))$ denotes the computational complexity of Newton's method used in RSFKM for each iteration. It can be seen that the complexity of REFCMFS is linear to $n$ and more suitable for handling the big dataset compared to GMM-based and graph-based methods.
%-------------------------------------------------
\begin{table}[t]
\renewcommand\tabcolsep{7pt}
\renewcommand\arraystretch{1.2}
\centering
\caption{The computational complexity of different methods.}
\begin{tabular}{llll}
\toprule
Methods & Complexity & Methods & Complexity \\
\midrule
K-Means  & $O(ncd)$ & GMM & $O(n^3cd)$ \\
K-Means++  & $O(ncd)$ & SC  & $O(n^3)$\\
K-Medoids & $O((n\!-\!c)^2cd)$ & RSFKM & $O(n(M(c)\!+\!dc))$\\
FCM & $O(nc^2d)$ & \textbf{REFCMFS} & $\bm{O(nc\cdot\max(d,c))}$\\
\bottomrule
\end{tabular}
\label{computational_complexity}
\end{table}

\subsection{Convergence Analysis}
To proof the convergence of the Algorithm \ref{alg2}, we need the \textbf{Lemma 1} proposed in \cite{nie2010efficient} to be used for the proof of \textbf{Theorem 1}. It can be listed as follows:
\begin{lemma}
For any nonzero vectors $\bm{u}_{t+1}$,$\bm{u}_t\in\mathbb{R}^d$, the following inequality holds:
%--------------------------
\begin{equation}
\|\bm{u}_{t+1}\|_2-\frac{\|\bm{u}_{t+1}\|_2^2}{2\|\bm{u}_t\|_2}\leq\|\bm{u}_t\|_2-\frac{\|\bm{u}_t\|_2^2}{\|\bm{u}_t\|_2}
\end{equation}
where $\bm{u}_{t+1}$ and $\bm{u}_t$ denote the results at the $t+1$-th and $t$-th iterations, respectively.
\end{lemma}
%-----------------------------
\begin{theorem}
The Algorithm \ref{alg2} monotonically decreases the objective of the problem (\ref{our_objective_function}) in each iteration and converges to the global optimum.
\end{theorem}
%-----------------------------
\begin{proof}

We decompose the problem (\ref{our_objective_function}) into two subproblems and utilize an alternately iterative optimization method to solve them.

According to \cite{boyd2004convex}, it is known that $f(x)=x^a$ is convex on $\mathbb{R}_{++}$ when $a\geq1$ or $a\leq0$, where $\mathbb{R}_{++}$ denotes the set of positive real numbers. For updating $\bm{\alpha}$, with the centroid matrix $\bm{B}$ fixed, the objective function of problem (\ref{alpha_i_h_new}) is $f_k(\alpha_{i,\psi(k)})\!=\!h_{i,\psi(k)}\alpha_{i,\psi(k)}^r$, where $r\!>\!1$, $k\!=\!1,\cdots,\tilde{K}$ and $h_{i,\psi(k)}$ can be seen as a constant. Therefore, $f_k(\alpha_{i,\psi(k)})$ is convex on $\mathbb{R}_{++}$ when $r\!>\!1$ and then $\sum_{k=1}^{\tilde{K}}f_k(\alpha_{i,\psi(k)})$ is convex when $r\!>\!1$.

For updating $\bm{B}$, with the membership matrix $\bm{\alpha}$ fixed, we use the \textbf{Lemma 1} to analyze the lower bound. After $t$ iterations, there are $\bm{B}^{(t)}$ and $s_{ik}^{(t)}$. Supposed that the updated $\bm{B}^{(t+1)}$ and $s_{ik}^{(t+1)}$ are the optimal solutions of problem (\ref{rewrite_B}), according to the definition of $s_{ik}$ in equation (\ref{s_ik}), there is:
%-------------------------------------
\begin{equation}
\sum_{i=1}^n\sum_{k=1}^c\frac{\|\bm{x}_i-\bm{b}_k^{(t+1)}\|_2^2}{2\|\bm{x}_i-\bm{b}_k^{(t)}\|_2}\alpha_{ik}^r
\!\leq\!
\sum_{i=1}^n\sum_{k=1}^c\frac{\|\bm{x}_i-\bm{b}_k^{(t)}\|_2^2}{2\|\bm{x}_i-\bm{b}_k^{(t)}\|_2}\alpha_{ik}^r
\label{B_inequality1}
\end{equation}
%-------------------------------------
According to the \textbf{Lemma 1}, we can obtain:
%-------------------------------
\begin{equation}
\begin{split}
&\sum_{i=1}^n\sum_{k=1}^c\left(\|\bm{x}_i-\bm{b}_k^{(t+1)}\|_2-\frac{\|\bm{x}_i-\bm{b}_k^{(t+1)}\|_2^2}{2\|\bm{x}_i-\bm{b}_k^{(t)}\|_2}\right)\alpha_{ik}^r   \\
&\leq
\sum_{i=1}^n\sum_{k=1}^c\left(\|\bm{x}_i-\bm{b}_k^{(t)}\|_2-\frac{\|\bm{x}_i-\bm{b}_k^{(t)}\|_2^2}{2\|\bm{x}_i-\bm{b}_k^{(t)}\|_2}\right)\alpha_{ik}^r
\end{split}
\label{B_inequality2}
\end{equation}
%------------------------
Combining inequalities (\ref{B_inequality1}) and (\ref{B_inequality2}), we can obtain:
%--------------------------------
\begin{equation}
\sum_{i=1}^n\sum_{k=1}^c\|\bm{x}_i-\bm{b}_k^{(t+1)}\|_2\alpha_{ik}^r
\leq
\sum_{i=1}^n\sum_{k=1}^c\|\bm{x}_i-\bm{b}_k^{(t)}\|_2\alpha_{ik}^r
\label{B_inequality3}
\end{equation}
%----------------------------------
which means that the problem (\ref{B}) has a lower bound.
Thus, in each iteration, Algorithm \ref{alg2} can monotonically decrease the objective function values of problem (\ref{our_objective_function}) until the algorithm converges.
\end{proof}

\section{Experiments}\label{experiments}

In this section, extensive experiments on several public datasets are conducted to evaluate the effectiveness of our proposed REFCMFS method.

\subsection{Experimental Setting}

\subsubsection{Datasets}
Several public datasets are used in our experiments which are described as follows.

\noindent\textbf{ORL.} This dataset \cite{samaria1994parameterisation} consists of 40 different subjects, 10 images per subject and each image is resized to 32$\times$32 pixels. The images are taken against a dark homogeneous background with the subjects in an upright, frontal position.

\noindent\textbf{Yale.} The dataset \cite{belhumeur1997eigenfaces} contains 165 gray-scale images of 15 individuals. There are 11 images per subject, one per different facial expression or configuration, and each image is resized to 32$\times$32 pixels.

\noindent\textbf{COIL20.} The dataset \cite{nene1996columbia} is constructed by 1440 gray-scale images of 20 objects (72 images per object). The size of each image is 32x 32 pixels, with 256 grey levels per pixel. The objects are placed on a motorized turntable against a black background and their Images are taken at pose intervals of 5 degrees.

\noindent\textbf{USPS.} The dataset \cite{hull1994database} consists of 9298 gray-scale handwritten digit images and each image is 16$\times$16 gray-scale pixels. It is generated by an optical character recognition algorithm which is used to scan 5 digit ZIP Codes and converts them to digital digits.

\noindent\textbf{YaleB.} For this database \cite{georghiades2001few}, it has 38 individuals and around 64 near frontal images under different illuminations per individual. We simply use the cropped images and resize them to 32$\times$32 pixels.

\noindent\textbf{COIL100.} This dataset \cite{nayar1996columbia} consists of 7200 color images of 100 objects. Similar to COIL20 dataset, the objects are placed on a motorized turntable against a black background and their images are taken at pose intervals of 5 degrees corresponding to 72 images per object.

\subsubsection{Compared methods}
We make comparisons between REFCMFS and several recent methods which are listed as follows.
K-Means clustering (K-Means),  Fuzzy C-Means clustering (FCM) \cite{bezdek1980convergence}, Spectral Clustering (SC) \cite{von2007tutorial}, and Gaussian Mixed Model (GMM) \cite{bishop2006pattern} are the baselines in our experiments. K-Means++ \cite{arthur2007k} and K-Medoids \cite{park2009simple} are the variants of K-Means clustering, where K-Means++ uses a fast and simple sampling to seed the initial centers for K-Means and K-Medoids replaces the mean with the medoid to minimize the sum of dissimilarities between the center of a cluster and other cluster members. Landmark-based Spectral Clustering (LSC) \cite{Chen11LandmarkSpectral} selects a few representative data points as the landmarks and represents the remaining data points as the linear combinations of these landmarks, where the spectral embedding of the data can then be efficiently computed with the landmark based representation, which can be applied to cluster the large-scale datasets. Robust and Sparse Fuzzy K-Means Clustering (RSFKM) \cite{xu2016robust} improves the membership matrix with proper sparsity balanced by a regularization parameter. Besides, we make a comparison between REFCMFS and its simplified version sim-REFCMFS which replaces the $L_{2,1}$-norm loss of REFCMFS with the least square criteria loss.

\subsubsection{Evaluation Metrics}
In our experiments, we adopt clustering accuracy (ACC) and normalized mutual information (NMI) as evaluation metrics. For these two metrics, the higher value indicates better clustering quality. Each metric penalizes or favors different properties in the clustering, and hence we report results on these two measures to perform a comprehensive evaluation.
%---------------------------------------
\begin{figure*}[t]
\begin{center}
\includegraphics[width=0.8\linewidth]{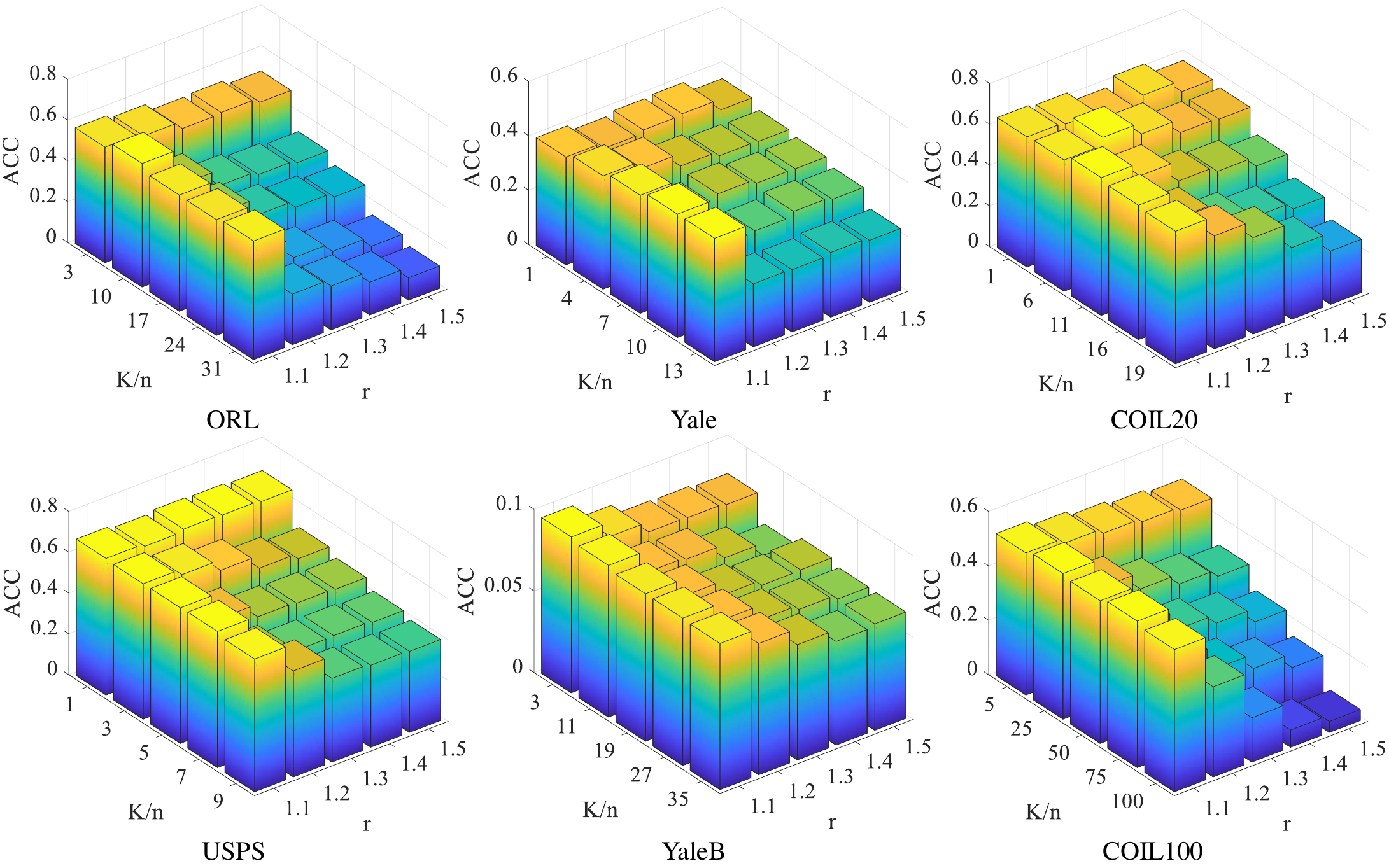}
\end{center}
\caption{Parameters $\tilde{K}=\frac{K}{n}$ and $r$ sensitivity analyses of REFCMFS on ORL, Yale, COIL20, USPS, YaleB, and COIL100 datasets according to Clustering Accuracy.}
\label{visualization-parameters-acc}
\end{figure*}
%----------------------------------------------
\begin{figure*}[t]
\begin{center}
\includegraphics[width=0.8\linewidth]{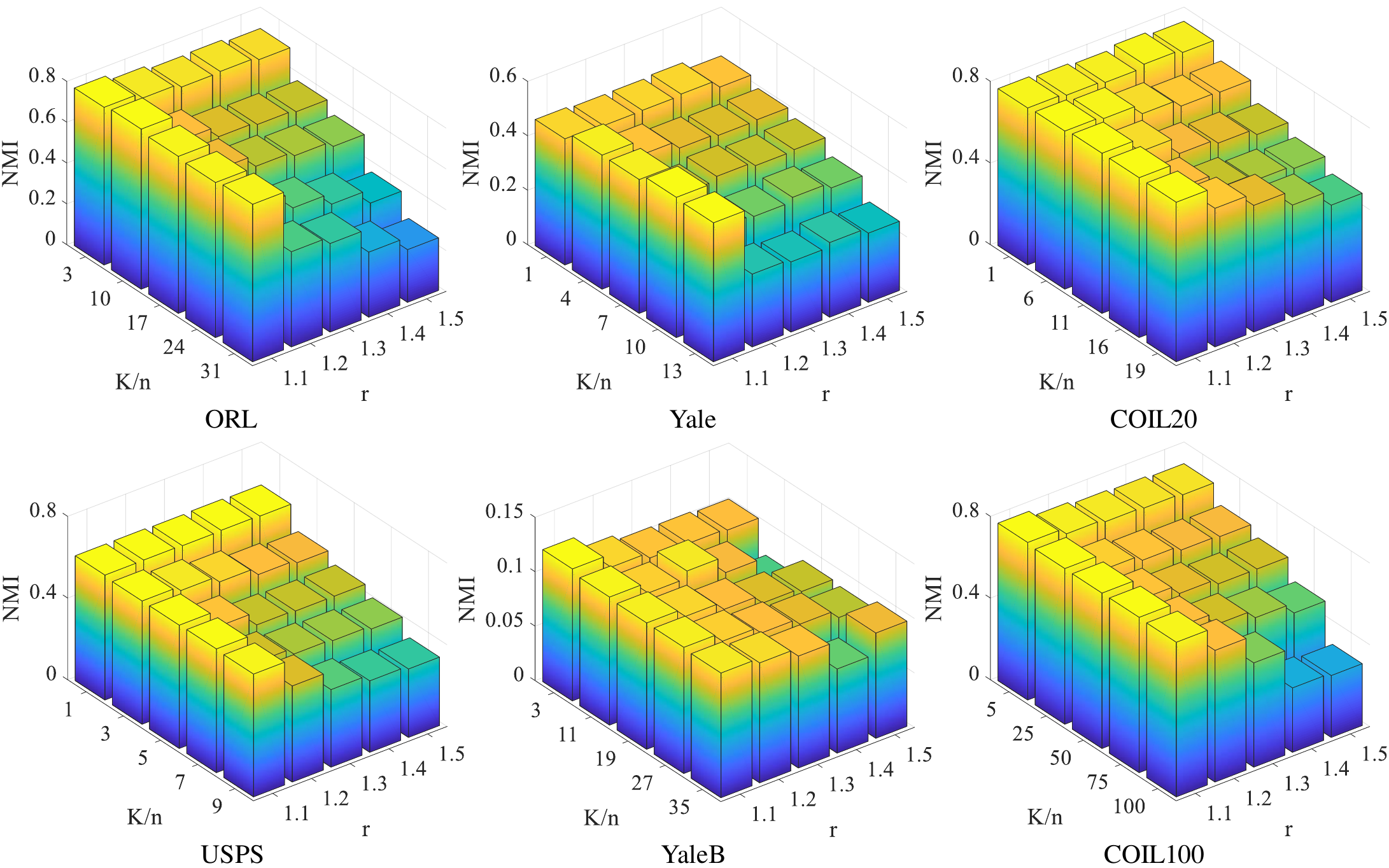}
\end{center}
\caption{Parameters $\tilde{K}=\frac{K}{n}$ and $r$ sensitivity analyses of REFCMFS on ORL, Yale, COIL20, USPS, YaleB, and COIL100 datasets according to Clustering NMI.}
\label{visualization-parameters-nmi}
\end{figure*}

\noindent\textbf{ACC.}
Let $r_i$ be the clustering result and $s_i$ be the ground truth label of $\bm{x}_i$. ACC is defined as:
%-------------------------------------------------------------
\begin{equation}
\text{ACC}=\frac{\sum_{i=1}^n\delta(s_i,map(r_i))}{n}
\label{acc}
\end{equation}
%-------------------------------------------------------------
Here $n$ is the total number of samples. $\delta(x,y)$ is the delta function that equals one if $x=y$ and equals zero otherwise. $map(r_i)$ is the best mapping function that utilizes the Kuhn-Munkres algorithm to permute clustering labels to match the ground truth labels.

\noindent\textbf{NMI.}
Supposing $C$ indicates the set of clusters obtained from the ground truth and $C'$ indicates the set of clusters obtained from our algorithm. Their mutual information metric $\text{MI}(C,C')$ is defined as:
%-------------------------------------------------------------------------
\begin{equation}
\text{MI}(C,C')=\underset{c_i\in C,c'_j\in C'}{\sum}p(c_i,c'_j)\text{log}_2\frac{p(c_i,c'_j)}{p(c_i)p(c'_j)}
\label{mi}
\end{equation}
%-------------------------------------------------------------------------
Here, $p(c_i)$ and $p(c'_j)$ are the probabilities that an arbitrary sample belongs to the clusters $c_i$ and $c'_j$, respectively. $p(c_i,c'_j)$ is the joint probability that the arbitrarily selected sample belongs to both the clusters $c_i$ and $c'_j$. Here, the following normalized mutual information (NMI) is adopted:
%---------------------------------------------------------------------------
\begin{equation}
\text{NMI}(C,C')=\frac{\text{MI}(C,C')}{max(H(C),H(C'))}
\label{nmi}
\end{equation}
%--------------------------------------------------------------------------
where $H(C)$ and $H(C')$ are the entropies of $C$ and $C'$, respectively. Note that $\text{NMI}(C,C')$ is ranged from 0 to 1. $\text{NMI}\!=\!1$ when the two sets of clusters are identical, and $\text{NMI}\!=\!0$ when they are independent.
%-------------------------------------------------
\begin{table*}[t]
\renewcommand\tabcolsep{25pt}
\renewcommand\arraystretch{1}
\centering
\caption{Comparison results on ORL and Yale datasets in terms of ACC and NMI.}
\begin{tabular}{lcccc}
\toprule
\multirow{2}{*}{Methods} &
\multicolumn{2}{c}{ORL} &
\multicolumn{2}{c}{Yale}\\
 & ACC & NMI & ACC & NMI \\
\midrule
\multicolumn{1}{l}{K-Means} & 48.60$\pm$1.40 & 71.28$\pm$1.37 & 42.91$\pm$6.18 & 49.66$\pm$3.41 \\
\multicolumn{1}{l}{K-Means++}  & 50.70$\pm$3.80 & 73.57$\pm$2.17 & 37.45$\pm$4.97 & 45.59$\pm$3.33 \\
\multicolumn{1}{l}{K-Medoids} & 42.10$\pm$2.90 & 63.03$\pm$1.84 & 37.58$\pm$3.64 & 43.51$\pm$3.46 \\
\multicolumn{1}{l}{GMM} & 53.85$\pm$5.40 & 75.53$\pm$2.48 & 40.61$\pm$4.24 & 48.27$\pm$3.80 \\
\multicolumn{1}{l}{SC}  & 37.55$\pm$1.55 & 66.75$\pm$0.86 & 24.79$\pm$0.48 & 35.79$\pm$0.50 \\
\multicolumn{1}{l}{LSC}  & 53.70$\pm$5.45 & 75.08$\pm$2.40 & 42.18$\pm$5.21 & 48.50$\pm$3.53 \\
\multicolumn{1}{l}{FCM} & 19.30$\pm$2.30 & 44.50$\pm$2.26 & 24.36$\pm$3.52 & 30.36$\pm$3.18 \\
\multicolumn{1}{l}{RSFKM}  & 53.94$\pm$4.31 & 75.84$\pm$1.41 & 39.15$\pm$5.82 & 46.05$\pm$1.97 \\
\multicolumn{1}{l}{sim-REFCMFS}   & 57.85$\pm$4.35 & 77.41$\pm$1.77 & 46.18$\pm$2.54 & 53.55$\pm$1.68 \\
\multicolumn{1}{l}{\textbf{REFCMFS}}  & \textbf{60.50}$\pm$\textbf{2.95} & \textbf{78.41}$\pm$\textbf{2.06} & \textbf{47.88}$\pm$\textbf{2.43} & \textbf{54.16}$\pm$\textbf{1.57} \\
\bottomrule
\end{tabular}
\label{results0}
\end{table*}
%-------------------------------------------------
\begin{table*}[t]
\renewcommand\tabcolsep{25pt}
\renewcommand\arraystretch{1}
\centering
\caption{Comparison results on COIL20 and USPS datasets in terms of ACC and NMI.}
\begin{tabular}{lcccc}
\toprule
\multirow{2}{*}{Methods} &
\multicolumn{2}{c}{COIL20} &
\multicolumn{2}{c}{USPS} \\
 & ACC & NMI & ACC & NMI \\
\midrule
\multicolumn{1}{l}{K-Means} & 52.99$\pm$6.53 & 73.41$\pm$1.92 & 64.30$\pm$3.08 & 61.03$\pm$0.62\\
\multicolumn{1}{l}{K-Means++}  & 57.74$\pm$5.60 & 75.79$\pm$2.52 & 64.05$\pm$2.35 & 60.79$\pm$0.54 \\
\multicolumn{1}{l}{K-Medoids} & 50.15$\pm$6.19 & 65.34$\pm$1.77 & 51.05$\pm$8.79 & 43.70$\pm$7.01 \\
\multicolumn{1}{l}{GMM} & 59.83$\pm$3.50 & 75.51$\pm$0.71 & 68.83$\pm$2.49 & 72.40$\pm$1.33 \\
\multicolumn{1}{l}{SC}  & 57.83$\pm$3.25 & 75.66$\pm$1.18 & 25.98$\pm$0.06 & 10.18$\pm$0.07 \\
\multicolumn{1}{l}{LSC}  & 59.76$\pm$4.21 & 73.43$\pm$3.29 & 62.42$\pm$4.23 & 58.39$\pm$2.16 \\
\multicolumn{1}{l}{FCM} & 23.85$\pm$4.62 & 41.31$\pm$3.88 & 37.79$\pm$2.45 & 29.71$\pm$2.68 \\
\multicolumn{1}{l}{RSFKM}  & 65.76$\pm$7.99 & 76.12$\pm$2.53 & 67.38$\pm$0.01 & 61.68$\pm$0.01 \\
\multicolumn{1}{l}{sim-REFCMFS}   & 68.56$\pm$5.13 & 76.36$\pm$2.39 & 67.56$\pm$6.94 & 61.37$\pm$1.62 \\
\multicolumn{1}{l}{\textbf{REFCMFS}}  & \textbf{69.51}$\pm$\textbf{3.40} & \textbf{77.60}$\pm$\textbf{1.66} & \textbf{70.02}$\pm$\textbf{8.58} & \textbf{66.79}$\pm$\textbf{2.93} \\
\bottomrule
\end{tabular}
\label{results1}
\end{table*}
%------------------------------------------------
\begin{table*}[t]
\renewcommand\tabcolsep{25pt}
\renewcommand\arraystretch{1}
\centering
\caption{Comparison results on USPS, YaleB and COIL100 datasets in terms of ACC and NMI.}
\begin{tabular}{lcccc}
\toprule
\multirow{2}{*}{Methods} &
\multicolumn{2}{c}{YaleB} &
\multicolumn{2}{c}{COIL100} \\
& ACC & NMI & ACC & NMI \\
\toprule
\multicolumn{1}{l}{K-Means}   & 9.36$\pm$0.00 & 12.34$\pm$0.00 & 48.21$\pm$2.62 & 75.69$\pm$0.67\\
\multicolumn{1}{l}{K-Means++}  & 9.55$\pm$0.76 & 13.04$\pm$1.29 & 46.26$\pm$0.68 & 75.58$\pm$0.27\\
\multicolumn{1}{l}{K-Medoids} & 6.68$\pm$0.37 & 8.25$\pm$0.20 & 31.59$\pm$0.68 & 63.10$\pm$0.50\\
\multicolumn{1}{l}{GMM} & 9.65$\pm$0.30 & 13.57$\pm$0.39 & 43.54$\pm$4.63 & 75.92$\pm$1.30\\
\multicolumn{1}{l}{SC}  & 7.71$\pm$0.21 & 10.03$\pm$0.39 & 8.91$\pm$0.19 & 26.72$\pm$0.13\\
\multicolumn{1}{l}{LSC}  & 9.56$\pm$0.82 & 12.46$\pm$1.32 & 48.05$\pm$1.84 & 75.71$\pm$0.91 \\
\multicolumn{1}{l}{FCM} & 7.49$\pm$0.65 & 9.91$\pm$1.28 & 10.43$\pm$1.56 & 42.16$\pm$2.83 \\
\multicolumn{1}{l}{RSFKM}  & 9.63$\pm$0.60 & 12.55$\pm$0.68 & 51.76$\pm$1.48 & 76.08$\pm$0.35 \\
\multicolumn{1}{l}{sim-REFCMFS}   & 9.88$\pm$1.30 & 12.79$\pm$0.73 & 52.60$\pm$1.46 & 76.45$\pm$0.32 \\
\multicolumn{1}{l}{\textbf{REFCMFS}}  & \textbf{10.04}$\pm$\textbf{0.47} & \textbf{13.61}$\pm$\textbf{0.61} & \textbf{53.15}$\pm$\textbf{1.84} & \textbf{77.82}$\pm$\textbf{0.74}\\
\bottomrule
\end{tabular}
\label{results2}
\end{table*}
%-------------------
\begin{table*}[ht]
\renewcommand\tabcolsep{18pt}
\renewcommand\arraystretch{0.9}
\centering
\caption{Time complexities of different methods on ORL, Yale, COIL20, USPS, YaleB, and COIL100 datasets, respectively.}
\begin{tabular}{lcccccc}
\toprule
\multirow{2}{*}{Methods} &
\multicolumn{6}{c}{Runtime (s)} \\
& ORL & Yale & COIL20 & USPS & YaleB & COIL100 \\
\midrule
K-Means  & 0.0917 & 0.0544 & 0.2501 & 0.3435 & 0.3245 & 1.6119\\
K-Means++  & 0.2262 & 0.0509 & 0.6452 & 1.3849 & 1.5643 & 12.7211\\
K-Medoids & 0.0362 & 0.0290 & 0.1386 & 2.2490 & 0.3049 & 2.3005\\
GMM & 6.7570 & 0.8921 & 178.5853 & 946.8664 & 298.7445 & 770.7133\\
SC & 1.6823 & 0.5310 & 14.2997 & 150.3779 & 39.5985 & 351.5918\\
FCM & 0.3669 & 0.1516 & 1.4901 & 0.9839 & 32.7830 & 18.4595\\
RSFKM & 0.5985 & 0.2768 & 2.6506 & 17.1141 & 6.9165 & 25.0760\\
\textbf{REFCMFS} & \textbf{0.3500} & \textbf{0.2395} & \textbf{1.6598} & \textbf{2.3491} & \textbf{6.3403} & \textbf{12.0973}\\
\bottomrule
\end{tabular}
\label{time_complexity}
\end{table*}
%------------------------------------------

\subsubsection{Parameter Setup}
There are two parameters $\tilde{K}$ and $r$ in our proposed REFCMFS method.
The first one $\tilde{K}$ in problem (\ref{subproblem_each_sample}) is utilized to adjust the number of nonzero elements in the membership vector $\bm{\alpha}^i$. We search the optimal $\tilde{\kappa}$ in the range of $(1,c)$ with different steps corresponding to different datasets. The second one $r$ in the problem (\ref{subproblem_each_sample}) controls how fuzzy the cluster will be (the higher the fuzzier) and can be tuned by a grid-search strategy from $1$ to $1.5$ with step $0.1$.
We report different values of parameters $\tilde{K}$ and $r$ in Figures \ref{visualization-parameters-acc} and \ref{visualization-parameters-nmi} to intuitively describe their sensitivity analyses of REFCMFS on different datasets, respectively, and record the best clustering results with optimal parameters.

It can be seen that each parameter plays an important role on the performance. Specifically, we set parameter $r\!=\!1.1$ and parameter $\tilde{K}$ for different datasets as follows: ORL $(\tilde{K}\!=\!10)$, Yale $(\tilde{K}\!=\!9)$, COIL20 $(\tilde{K}\!=\!13)$, USPS $(\tilde{K}\!=\!5)$, YaleB $(\tilde{K}\!=\!3)$, and COIL100 $(\tilde{K}\!=\!75)$. Taking the YaleB dataset as an example, the 3D bars of ACC and NMI simultaneously achieve the highest values when $\tilde{K}\!=\!3$ and $r\!=\!1.1$.
%------------------------------------------
\begin{figure*}[ht]
\begin{center}
\includegraphics[width=\linewidth]{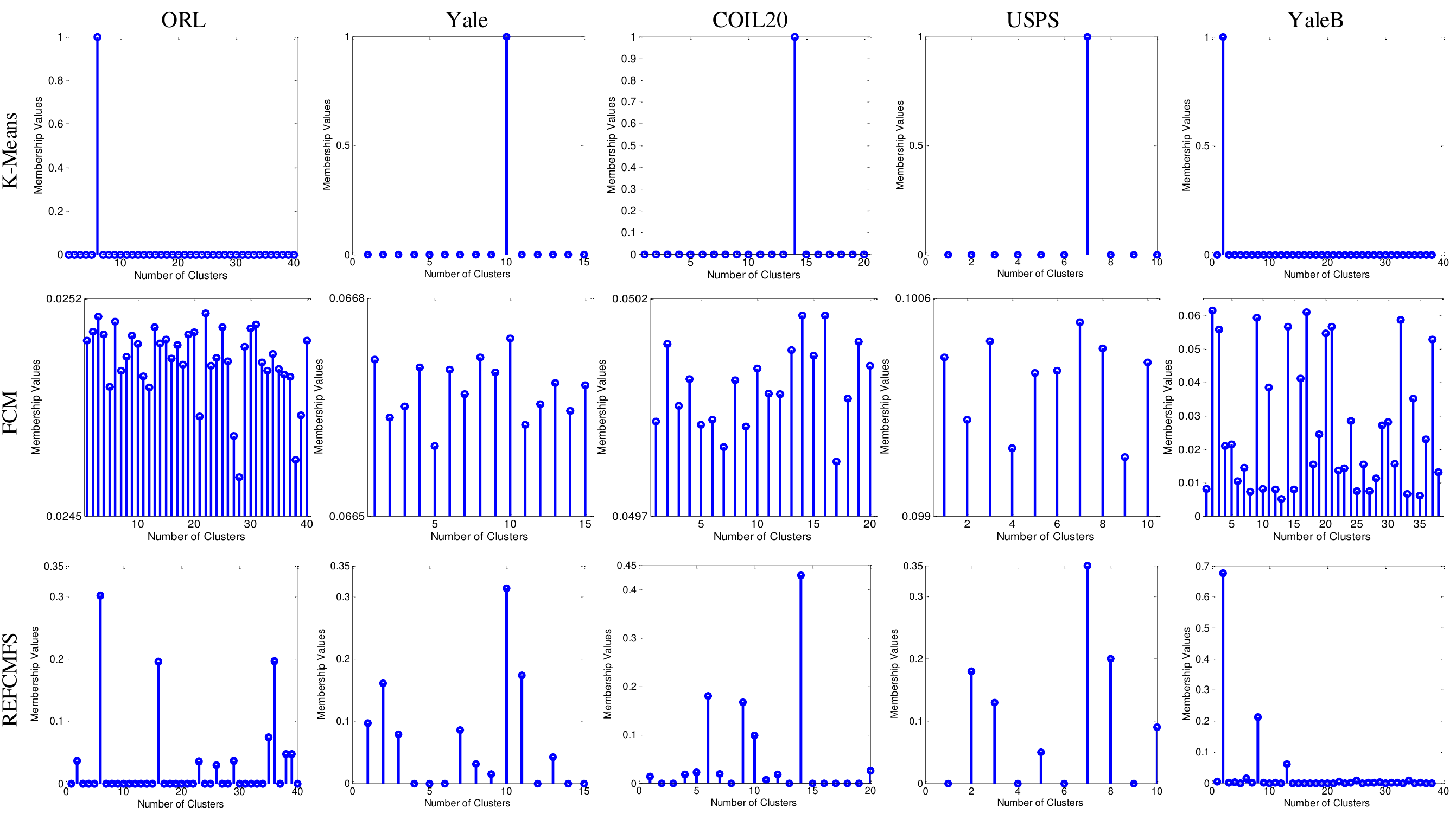}
\end{center}
\caption{The membership values of each sample for K-Means, FCM, and REFCMFS on ORL, Yale, COIL20, USPS, and YaleB datasets, respectively.}
\label{outA}
\end{figure*}

\subsection{Experimental Results}
In this section, we report the clustering performance comparisons of REFCMFS in Tables \ref{results0} $\sim$ \ref{results2} and have the following observations.

Compared to four baselines K-Means, FCM, SC, and GMM, our proposed REFCMFS method and its simple version sim-REFCMFS generally achieve better performance on all the datasets. For instance, on the Yale dataset, REFCMFS obtains 4.735\%, 23.66\%, 20.73\%, and 6.58\% average improvements (For simplicity, the average improvement here is defined as the improvement averaged over two clustering evaluation metrics ACC and NMI.), respectively, compared to four baselines. Similarly, sim-REFCMFS gains 3.58\%, 22.505\%, 19.575\%, and 5.425\% average improvements, respectively. This observation indicates that it is beneficial to combine the advantages of hard partition and soft partition, to introduce the $L_{2,1}$-norm robust loss, and to make the membership matrix with proper sparsity. This conclusion also can be demonstrated on other five datasets. Moreover, to intuitively present the flexible sparse membership values of REFCMFS with respect to those of the hard and soft partitions (i.e., K-Means and FCM), we show them in Figure \ref{outA}, which can be seen that flexible sparsity is more beneficial to clustering.

Besides, compared to K-Means++ and K-Medoids (two variants of K-Means), REFCMFS and sim-REFCMFS obtain better results on all the datastes. Specifically, for the COIL20 dataset, REFCMFS achieves 6.79\% and 15.81\% average improvements and sim-REFCMFS gets 5.695\% and 14.715\% average improvements. It is obvious that although K-Means++ and K-Medoids improve K-Means in initialization they are not good at handling outliers because of the poor robustness of the least squares criterion. This conclusion also can be verified on other five datasets. Concretely, compared with K-Means++, REFCMFS achieves 7.32\%, 9.5\%, 5.985\%, 0.53\%, and 4.565\% average improvements on ORL, Yale, USPS, YaleB, and COIL100 datasets, respectively, and sim-REFCMFS obtains 5.495\%, 8.345\%, 2.045\%, 0.07\% and 3.605\% average improvements. Compared to K-Medoids, REFCMFS achieves 16.89\%, 10.475\%, 21.03\%, 4.36\% and 18.14\% average improvements on ORL, Yale, USPS, YaleB, and COIL100 datasets, respectively, and sim-REFCMFS obtains 15.065\%, 9.32\%, 17.09\%, 3.87\% and 13.35\% average improvements.

In addition, REFCMFS outperforms two recent works LSC and RSFKM on all the datasets. Considered that LSC needs to select a few representative data points as the landmarks and represents the remaining data points as the linear combinations of these landmarks, how to select the representative information directly affects on this method.
Therefore, compared to LSC, REFCMFS achieves 5.065\%, 5.68\%, 6.96\%, 8\%, 0.815\%, and 3.605\% average improvements on ORL, Yale, COIL20, USPS, YaleB, and COIL100 datasets, respectively.
RSFKM introduces a penalized regularization on the membership matrix and controls the sparsity of membership matrix by the regularization parameter, which differs in that REFCMFS efficiently adjusts the sparsity of membership matrix through its $L_0$-norm constraint. Compared to RSFKM, REFCMFS achieves 4.565\%, 8.42\%, 2.615\%, 3.875\%, 0.735\%, and 1.365\% average improvements on ORL, Yale, COIL20, USPS, YaleB, and COIL100 datasets, respectively.

Furthermore, sim-REFCMFS is the simple version of REFCMFS, which achieves the second best performance on almost all datasets. Both sim-REFCMFS and REFCMFS prove that introducing the $L_0$-norm constraint with flexible sparsity imposed on the membership matrix can result in better performance than other comparison methods. Whereas, the loss function of sim-REFCMFS is based on the least square criteria, not a robust loss based on the $L_{2,1}$-norm, which may be sensitive to the outliers. Concretely, compared with sim-REFCMFS, REFCMFS achieves 1.825\%, 1.155\%, 1.095\%, 3.94\%, 0.49\%, 0.96\% average improvements on ORL, Yale, COIL20, USPS, YaleB, and COIL100 datasets, respectively.

Finally, Figure \ref{lossvalues} shows the convergence curves and proves above convergence analysis of REFCMFS.
What is more, combining the above computational complexity analysis in subsection \ref{computational_analysis}, we calculate the time complexities of different methods on all the datasets and report them in Table \ref{time_complexity}. It is obvious REFCMFS is faster than RSFKM, GMM, and SC methods on all the datasets.
%---------------------------------------------
\begin{figure*}[ht]
\begin{center}
\includegraphics[width=\linewidth]{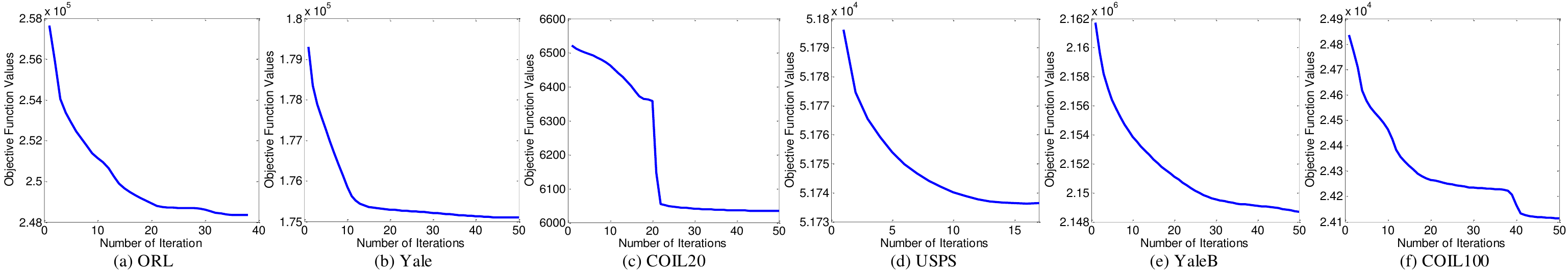}
\end{center}
\caption{The convergence curves of our proposed REFCMFS  method on ORL, Yale, COIL20, USPS, YaleB, and COIL100 datasets, respectively.}
\label{lossvalues}
\end{figure*}

\section{Conclusion}

In this paper, we have proposed a novel clustering algorithm, named REFCMFS, which develops a $L_{2,1}$-norm robust loss for the data-driven item and imposes a $L_0$-norm constraint on the membership matrix to make the model more robust and sparse flexibly. This not only avoids the incorrect or invalid clustering partitions from outliers but also greatly reduces the computational complexity. Concretely, REFCMFS designs a new way to simplify and solve the $L_0$-norm constraint directly without any approximate transformation by absorbing $\|\cdot\|_0$ into the objective function through a ranking function. This make REFCMFS can be solved by a tractable and skillful optimization method and guarantees the optimality and convergence. Theoretical analyses and extensive experiments on several public datasets demonstrate the effectiveness and rationality of our proposed method.

\bibliographystyle{IEEEtran}
\bibliography{jinglin-modified-version}

\end{document}